\documentclass{article} 
\usepackage[table]{xcolor}   
\usepackage{tencent_tech_report}
\usepackage[colorlinks = true,
            linkcolor = blue,
            urlcolor  = blue,
            citecolor = blue,
            anchorcolor = blue]{hyperref}

\usepackage{tcolorbox} 
\newtcolorbox{alprompt}[1]{
        boxrule = 1pt,
        fontupper = \small\tt,
        fonttitle = \bf\color{black},
        arc = 2pt,
        rounded corners,
        colframe = black,
        colbacktitle = white!97!yellow,
        colback = white!97!yellow,
        title = #1,
}

\usepackage{microtype}
\usepackage{hyperref}
\usepackage{url}
\usepackage{booktabs}
\usepackage{enumitem}
\usepackage{multicol}
\usepackage{multirow}
\usepackage{CJKutf8}
\usepackage{amsmath}
\usepackage{siunitx}
\usepackage{floatflt}
\usepackage{wrapfig}
\usepackage{graphicx}
\usepackage{booktabs}
\usepackage{authblk}
\usepackage{lipsum}

\usepackage{algorithm}
\usepackage{algorithmicx}
\usepackage{algpseudocode}
\usepackage{microtype}
\usepackage{graphicx}
\usepackage{multirow}
\usepackage{booktabs} 
\usepackage{pifont}  
\usepackage{graphicx}  
\usepackage{subcaption} 
\usepackage{hyperref}

\usepackage{amssymb} 

\algnewcommand{\LeftComment}[1]{\Statex \(\triangleright\) #1}

\usepackage{array}
\usepackage{amsmath}
\usepackage{amssymb}
\usepackage{mathtools}
\usepackage{amsthm}
\usepackage{arydshln}
\usepackage[capitalize,noabbrev]{cleveref}
\usepackage{adjustbox} 
\usepackage{enumitem}

\theoremstyle{plain}
\newtheorem{theorem}{Theorem}[section]

\newtheorem{lemma}[theorem]{Lemma}

\theoremstyle{definition}

\theoremstyle{remark}

\usepackage{makecell}

\definecolor{mygreen}{HTML}{E0FFE0} 
\definecolor{lightgray}{HTML}{EDEDED}
\definecolor{darkgreen}{rgb}{0.0, 0.5, 0.0} 

\usepackage[textsize=tiny]{todonotes}

\definecolor{nred}{RGB}{196, 38, 11}
\definecolor{ngreen}{RGB}{18, 141, 21}
\definecolor{nblue}{RGB}{41, 52, 190}
\definecolor{norange}{RGB}{230, 106, 53}

\newcommand{\ignore}[1]{}

\usepackage{tcolorbox}
\tcbuselibrary{listings,skins}

\definecolor{promptbg}{RGB}{245, 245, 245}   
\definecolor{promptborder}{RGB}{200, 200, 200} 

\tcbset{
    promptstyle/.style={
        enhanced,
        frame hidden,
        boxrule=0pt,
        left=8pt,
        right=8pt,
        top=6pt,
        bottom=6pt,
        arc=3pt,
        colback=promptbg,
        colframe=promptborder,
        fontupper=\fontsize{0.7em}{0.7em}\selectfont\ttfamily\leftskip,
        before upper={\parindent}, 
        overlay unbroken and first={
            \draw[promptborder, line width=1pt] 
                (frame.south west) -- (frame.north west);
            \draw[promptborder, line width=1pt] 
                ([xshift=-2pt]frame.north east) -- (frame.south east);
        },
        overlay middle and last={
            \draw[promptborder, line width=1pt] 
                (frame.south west) -- (frame.north west);
            \draw[promptborder, line width=1pt]
                ([xshift=-2pt]frame.north east) -- (frame.south east);
        }
    }
}

\usepackage{wrapfig}
\usepackage{multirow}

\colmfinalcopy

\newcommand{\method}{{VISTA}}

\title{VISTA: Enhancing Vision-Text Alignment in MLLMs via Cross-Modal Mutual Information Maximization}

\author[ ]{Mingxiao Li$^{*}$}
\author[ ]{Na Su}
\author[ ]{Fang Qu}
\author[ ]{Zhizhou Zhong}
\author[ ]{Ziyang Chen}
\author[ ]{Yuan Li}
\author[ ]{\\Zhaopeng Tu\thanks{Correspondence to: Mingxiao Li \textless mingxiaoli@tencent.com\textgreater~and Zhaopeng Tu \textless zptu@tencent.com\textgreater.}}
\author[ ]{Xiaolong Li}

\affil[ ]{Hunyuan AI Digital Human, Tencent 
\protect\\[2pt] 
\url{https://github.com/Tencent/DigitalHuman/tree/main/VISTA}}

\begin{document}

\maketitle

\begin{figure*}[h]
\vspace{-12pt}
\begin{center}
\centerline{\includegraphics[width=1.0\textwidth]{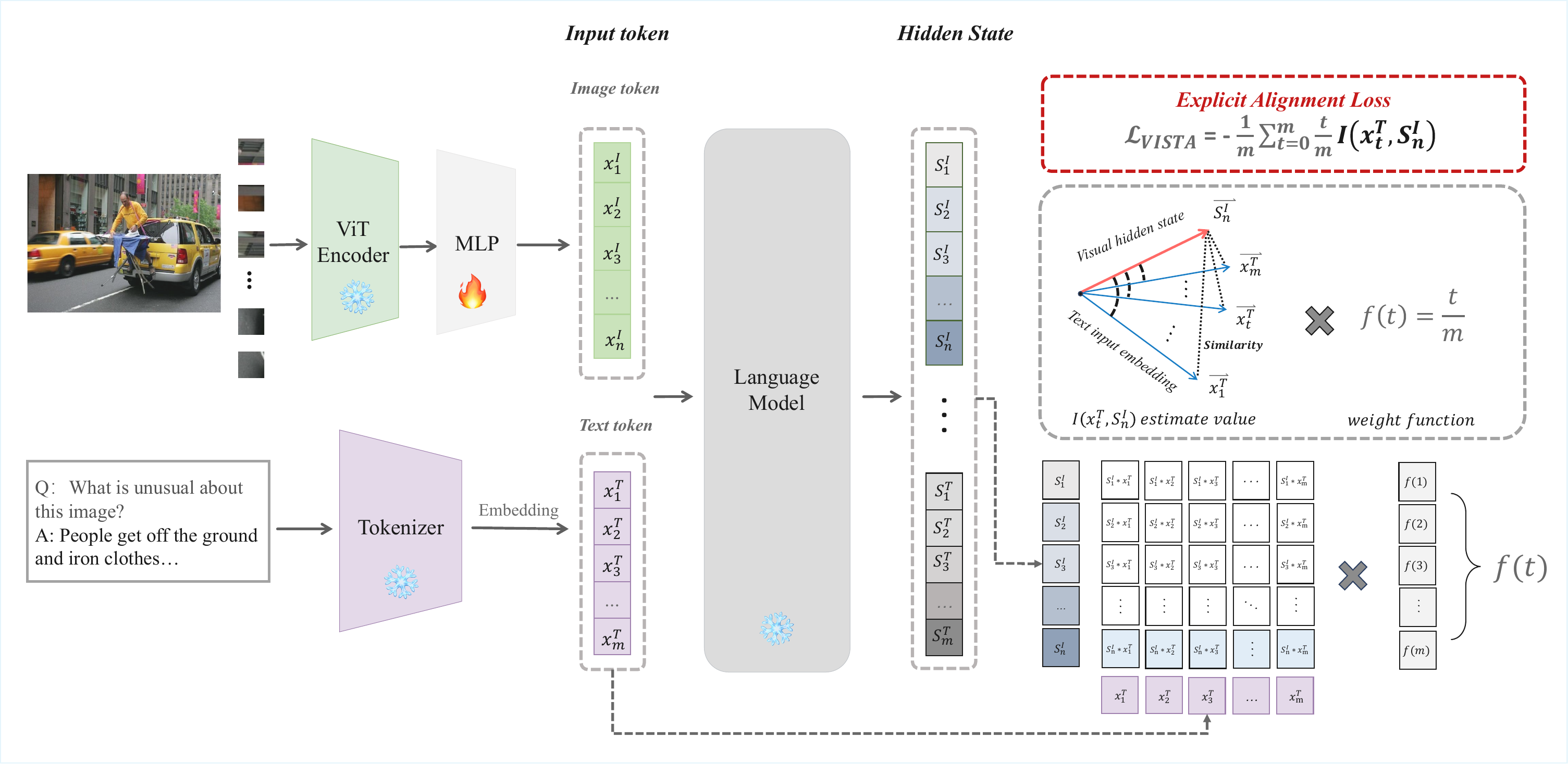}}
\caption{\textbf{\method{} framework}. We prevent target degradation caused by implicit alignment of the target by introducing the explicit alignment loss $\mathcal{L_{\text{VISTA}}}$of the text token and vision hidden state.}
\label{VISTA}
\end{center}
\end{figure*}

\begin{abstract}
Current multimodal large language models (MLLMs) face a critical challenge in modality alignment, often exhibiting a bias towards textual information at the expense of other modalities like vision. This paper conducts a systematic information-theoretic analysis of the widely used cross-entropy loss in MLLMs, uncovering its implicit alignment objective. Our theoretical investigation reveals that this implicit objective has inherent limitations, leading to a degradation of cross-modal alignment as text sequence length increases, thereby hindering effective multimodal information fusion. To overcome these drawbacks, we propose \textbf{Vis}ion-\textbf{T}ext \textbf{A}lignment (\method), a novel approach guided by our theoretical insights. \method{} introduces an explicit alignment objective designed to maximize cross-modal mutual information, preventing the degradation of visual alignment. 
Notably, \method{} enhances the visual understanding capabilities of existing MLLMs without requiring any additional trainable modules or extra training data, making it both efficient and practical. Our method significantly outperforms baseline models across more than a dozen benchmark datasets, including VQAv2, MMStar, and MME, paving the way for new directions in MLLM modal alignment research. 
\end{abstract}

\section{Introduction}

The rapid advancement of artificial intelligence has propelled expectations for AI assistants beyond traditional text-based interactions, driving demand for systems that align more closely with the inherently multimodal nature of human perception and communication. Multimodal Large Language Models (MLLMs) have emerged at the forefront of this pursuit, with architectures like LLaVA \citep{llava} pioneering a widely adopted paradigm. These models typically utilize a Multi-Layer Perceptron (MLP) to bridge pre-trained Vision Transformers (ViTs) and Large Language Models (LLMs), integrating visual and textual modalities. This ``Glue Structure'' approach, exemplified by influential models such as InternVL \citep{internvl1.5}, QwenVL \citep{qwenvl2}, and Kimi-VL \citep{kimivl}, facilitates rapid, cost-effective adaptation of the latest unimodal foundation models, enabling the swift development of new MLLMs and transcending the limitations of text-only systems to address complex multimodal problems.

Despite their architectural adaptability, prevalent MLLMs predominantly rely on standard cross-entropy loss applied to text generation for end-to-end training and alignment. This straightforward approach inherently prioritizes the accurate prediction of text tokens. Consequently, it often leads to two critical issues: (1) {\bf a bias towards textual information}, potentially diminishing the significance of visual context in semantic understanding, and (2) {\bf an implicit and often weak vision-text alignment mechanism} in text-only cross-entropy loss that lacks explicit constraints to enforce semantic consistency between visual and textual representations. Recent studies \citep{Opera, zhang2024redundancy,liu2024conflicts} corroborate this, highlighting that existing MLLMs frequently exhibit a modality imbalance, with their representational capacity heavily skewed towards textual information.

Several attempts have been made to mitigate this imbalance by incorporating forms of image-focused supervision. For instance, ROSS \citep{ross} introduces auxiliary learnable modules to reconstruct input images during alignment training, thereby enhancing fine-grained visual perception. Similarly, Metamorph \citep{Metamorph} appends a trainable MLP to reconstruct visual semantic vectors from the ViT, aiming to create a unified MLLM for both generation and comprehension. While these methods have shown improvements, they often introduce additional trainable parameters, require architectural modifications, or necessitate extra training data and stages, and may not directly address the core issue of {\bf implicit cross-modal semantic alignment}.

The limitations of current implicit alignment paradigms are starkly illustrated in Figure \ref{fig:llava}, where visualizations of image-text embedding similarity in models like LLaVA reveal considerable semantic noise between modalities. This underscores a fundamental weakness: {\bf the absence of a direct, explicit objective to align visual and textual representations in a shared semantic space}. In contrast, foundational models like CLIP \citep{clip} have demonstrated the power of explicit alignment by directly optimizing the correspondence between image and text embeddings. Recent work \citep{covert2024locality} further reinforces the potential of explicit alignment strategies by employing masked autoencoder-style post-training on ViTs to enhance local-region awareness within the visual-textual semantic space. This suggests that incorporating explicit alignment objectives is a promising avenue for overcoming the current limitations of MLLMs.

\begin{figure}[t]
    \centering
    \begin{minipage}[b]{0.3\textwidth}
        \centering
        \includegraphics[width=\textwidth]{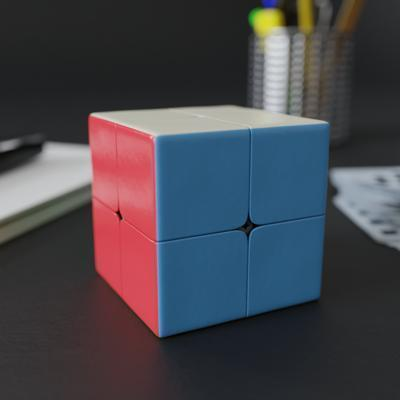}
        \subcaption{Origin}
    \end{minipage}
    \hfill
    \begin{minipage}[b]{0.3\textwidth}
        \centering 
        \includegraphics[width=\textwidth]{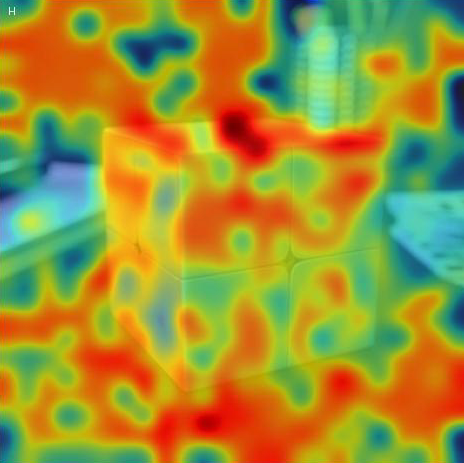}
        \subcaption{Vanilla LLaVA} 
        \label{fig:llava}
    \end{minipage}
    \hfill
    \begin{minipage}[b]{0.3\textwidth}
        \centering
        \includegraphics[width=\textwidth]{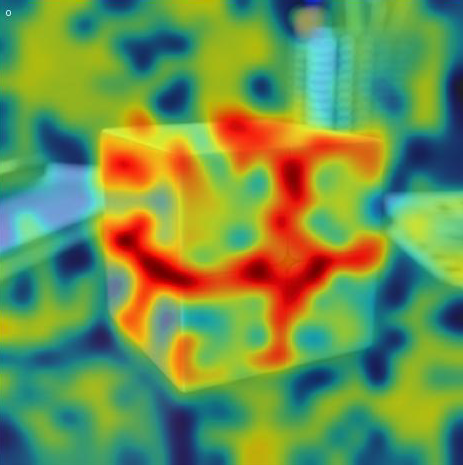}
        \subcaption{LLaVA-VISTA}
        \label{fig:vista}
    \end{minipage}
    \caption{Visualization of the semantic similarity between prompt text embeddings and image embeddings generated by the MLP. The prompt used is: \textit{"How do you solve this specific type of puzzle cube depicted in the image?"} Similarity is measured using cosine similarity with mean aggregation. Darker colors indicate higher semantic alignment.}
    \label{fig:similarity}
\end{figure}

Motivated by these observations, this paper revisits the conventional cross-entropy-based alignment framework from an information-theoretic perspective. Our analysis uncovers that while cross-entropy implicitly encourages alignment between image hidden states and corresponding text tokens, this implicit objective inherently degrades as the length of the text sequence increases, particularly in decoder-only architectures. The contribution of visual information to the overall learning objective diminishes, hindering robust multimodal fusion.

To address this fundamental limitation, we propose {\bf Vis}ion-{\bf T}ext {\bf A}lignment (\method), a novel approach that introduces an explicit alignment objective designed to maximize cross-modal mutual information. Specifically, \method{} incorporates an explicit alignment loss coupled with a weighting function. This combination ensures that the alignment objective maintains a consistent contribution to the overall optimization process, preventing its dilution as text length increases and thereby ensuring robust cross-modal alignment.
Crucially, \method{} enhances the visual understanding capabilities of existing MLLMs without requiring any additional trainable modules, architectural modifications, or extra training data, making it an efficient and practical solution. As illustrated in Figure \ref{fig:vista}, \method{} effectively fosters semantic coherence across modalities. 

We evaluate \method{} on two representative MLLM backbones, TinyLLaVA (3B) and LLaVA-v1.5 (7B), across over a dozen diverse benchmarks. These benchmarks encompass high-level Visual Question Answering (VQA), general multimodal reasoning, and fine-grained visual perception. Our approach consistently yields an average improvement exceeding 2\% on these benchmarks, without introducing additional parameters or requiring further tuning. Notably, \method{} achieves substantial gains on particularly challenging datasets—including RealWordQA (+2.8\%), MMStar (+7.17\%), and MME Cognition (+8.5\%)—underscoring its robustness and scalability across different model scales and vision-language tasks.

Our key contributions are as follows:
\begin{itemize}[leftmargin=12pt]
    \item We provide a novel information-theoretic analysis of the conventional cross-entropy loss in MLLMs, theoretically demonstrating how its implicit cross-modal alignment capability degrades with increasing text sequence length.
    \item We introduce Vision-Text Alignment (\method), a principled and lightweight method that directly maximizes cross-modal mutual information to enhance vision-language alignment, without requiring auxiliary trainable modules, architectural modifications, or additional training data.
    \item We extensively evaluate \method{} across more than a dozen diverse and representative benchmarks, showing significant and consistent performance improvements over strong MLLM baselines, thereby validating its practical effectiveness and efficiency.
\end{itemize}

\section{Related work}
\label{sec:related}

\textbf{Architectures of multimodal large language models.}  
The design of multimodal large language models (MLLMs) has evolved from early query-based cross-modal interaction frameworks, exemplified by MiniGPT-4~\cite{zhu2023minigpt} and BLIP-2~\cite{li2023blip2}, which utilize modules like Q-Former to mediate between visual and textual modalities. A paradigm shift occurred with LLaVA~\cite{llava}, introducing a streamlined encoder-MLP-LLM pipeline that projects frozen vision encoder (e.g., ViT~\cite{vit}) features into the LLM embedding space via a lightweight MLP. This architecture has become the prevailing standard, adopted by subsequent models such as InternVL~\cite{internvl1.5}, QwenVL~\cite{qwenvl2,qwenvl2-5}, and KimiVL~\cite{kimivl}. Although these methods input image tokens as prompts, they still optimize only the LLM's text-only cross-entropy loss. Thus, from the decoder's perspective, alignment remains implicit, with no explicit supervision linking image tokens to text.

\textbf{Challenges in vision-language alignment.} 
Recent works reveal fundamental limitations in current alignment strategies, particularly their failure to preserve detailed visual semantics within MLLMs. For instance, InternVL-1.5~\cite{internvl1.5} showed that scaling vision encoders from scratch significantly improves visual grounding, while InternVL-3~\cite{zhu2025internvl3} enhanced semantic integration by co-training ViT and LLM parameters. Auxiliary modules like Metamorph~\cite{Metamorph}, AIMV2~\cite{apple}, and Ross~\cite{ross} further encourage visual concept internalization or semantic consistency via MLP reconstruction, pixel alignment, or diffusion-based modules. Although empirically effective, these methods rely on task-specific modules, extensive end-to-end training, or architectural changes, lacking a unified theoretical framework and often overlooking the mismatch between MLLMs' instruction-following objectives and the semantic fidelity needed for robust visual understanding.

\textbf{Explicit alignment in multimodal pretraining.}  
Encoder-level explicit alignment has shown effectiveness in multimodal representation learning, exemplified by CLIP~\cite{clip}, which maximizes mutual information between image-text embeddings, and GRAM~\cite{GRAM}, which enforces geometric alignment by minimizing the Gram volume spanned by modality vectors, achieving state-of-the-art results on video-audio-text benchmarks. In contrast, MLLMs mainly depend on implicit alignment via autoregressive text generation, which remains insufficient to bridge the modality gap~\cite{Opera, lanp, li2025visuals}. Recent attempts to incorporate explicit alignment into MLLMs include Sugar~\cite{chow2024unified}, introducing a trainable MLP and interlaced image-text data but altering training data format, and a masked joint training approach~\citep{covert2024locality} optimizing encoder-level alignment through the text encoder. Despite progress, a principled, generalizable explicit alignment framework tailored for MLLMs remains an open challenge.

\section{Information-theoretic analysis of conventional cross-entropy loss}
\label{sec:background}

\subsection{Implicit cross-modal alignment via cross-entropy loss}

In multimodal large language models (MLLMs), the standard training objective is the autoregressive text cross-entropy loss conditioned on image tokens as prompts. Although this appears to align visual and textual modalities implicitly, the alignment is not explicitly enforced in the training target. In this section, we formalize this implicit alignment and analyze its limitations.

\paragraph{Notation system.}
Let $\mathbf{X}^I = \{x^I_1, \ldots, x^I_n\}$ denote the image token sequence and $\mathbf{X}^T = \{x^T_1, \ldots, x^T_m\}$ 
the text token sequence. At decoding step $t$, the model predicts $x^T_t$ conditioned on $\mathbf{X}^I$ and $x^T_{<t}$. 
The visual and textual hidden states are $\mathbf{S}^I = \{s^I_1, \ldots, s^I_n\}$ and $s^T_{<t}$ respectively, both in $\mathbb{R}^{d_v \times \cdot}$. The embedding dimension is $\mathcal{V}$. We denote entropy by $\mathcal{H}(\cdot)$ and mutual information by $\mathcal{I}(\cdot;\cdot)$.

\paragraph{The nature of cross-entropy.}
The multimodal alignment objective of the original Multimodal Large Language Model (MLLM) at step $t$ is the negative log-likelihood, is actually equivalent to  minimizing conditional entropy:
\begin{equation}
\mathcal{L}_{CE} = - \mathbb{E}\left[\log p(x^T_{\geq t} \mid \mathbf{X}^I, x^T_{<t})\right] \sim - \mathcal{H}(x^T_{\geq t} \mid \mathbf{X}^I, s^T_{<t}).
\end{equation}

Since decoder-only large language models (LLMs) generate tokens based on their internal hidden states rather than directly on raw input tokens, the conditional probability of the next text token can be equivalently expressed in terms of these hidden representations. Specifically, let $\mathbf{S}^I$ denote the visual hidden states extracted from the image input $\mathbf{X}^I$, and let $s^T_{<t}$ denote the textual hidden states corresponding to the previously generated tokens $x^T_{<t}$. Then, we have:
\begin{equation}
p(x^T_t \mid \mathbf{X}^I, x^T_{<t}) = q(x^T_t \mid \mathbf{S}^I, s^T_{<t}),
\end{equation}
where $q(\cdot)$ represents the model's output distribution conditioned on the hidden states. This formulation highlights that the prediction of the next token depends solely on the internal hidden representations of both modalities, providing a foundation for analyzing multimodal alignment at the hidden state level. Substituting this into the cross-entropy loss yields:
\begin{equation}
\label{eq:hidden}
\mathcal{L}_{CE} =  - \mathbb{E}\left[\log q(x^T_{\geq t} \mid \mathbf{S}^I, x^T_{<t})\right]=- \frac{1}{m} \sum_{t=1}^m \mathcal{H}(x^T_t \mid \mathbf{S}^I, s^T_{<t}).
\end{equation}

According to formula~\ref{eq:hidden},  Minimizing this conditional entropy is equivalent to maximizing the mutual information:
\begin{equation}
\max \mathcal{I}(x^T_t; \mathbf{S}^I, s^T_{<t}) = \max \left[ \mathcal{I}(x^T_t; \mathbf{S}^I) + \mathcal{I}(x^T_t; s^T_{<t} \mid \mathbf{S}^I) \right].
\end{equation}

Here, $\mathcal{I}(x^T_t; \mathbf{S}^I)$ reflects visual-textual alignment, while $\mathcal{I}(x^T_t; s^T_{<t} \mid \mathbf{S}^I)$ captures textual autoregressive dependencies. This decomposition reveals that the cross-entropy loss only implicitly encourages modality alignment.

\subsection{Vanish of visual contribution as the text length grows}

We now theoretically prove that the implicit alignment of cross-entropy objective increasingly favors textual context over visual information as generation proceeds.

\begin{lemma}[Autoregressive entropy growth]
\label{lem:entropy_growth}

For a non-degenerate autoregressive text model, the entropy of the text prefix grows at least linearly with length:

\begin{equation}
\mathcal{H}(x^T_{<t}) = \sum_{k=1}^{t-1} \mathcal{H}(x^T_k \mid x^T_{<k}) \geq t \Delta_H, 
\quad \text{where} \Delta_H = \min_{k} \mathcal{H}(x^T_k \mid x^T_{<k}) > 0
\end{equation}

\end{lemma}

\begin{proof}
By the chain rule of entropy and non-degeneracy assumption, each conditional entropy term is lower bounded by $\Delta_H$. Summing over $t-1$ terms yields the linear growth.
\end{proof}

\begin{lemma}[Bounded visual mutual information]
\label{lem:visual_bound}

The mutual information between the visual hidden states and the generated token is upper bounded:

\begin{equation}
\mathcal{I}(x^T_t; \mathbf{S}^I) \leq \mathcal{H}(x^T_t) \leq C,
\end{equation}

where $C$ is a constant determined by the token entropy.
\end{lemma}

\begin{proof}
By definition, $\mathcal{I}(x^T_t; \mathbf{S}^I) = \mathcal{H}(x^T_t) - \mathcal{H}(x^T_t \mid \mathbf{S}^I) \leq \mathcal{H}(x^T_t)$. Since $x^T_t$ is discrete with bounded vocabulary, its entropy is upper bounded by $C$.
\end{proof}

\begin{theorem}[Vanishing vision-text alignment contribution]
\label{thm:vanishing_contribution}

As the text length $t$ grows, the relative contribution of visual alignment to total mutual information vanishes:

\begin{equation}
\lim_{t \to \infty} \frac{\mathcal{I}(x^T_t; \mathbf{S}^I)}{\mathcal{I}(x^T_t; \mathbf{S}^I, s^T_{<t})} \leq \lim_{t \to \infty} \frac{C}{\mathcal{I}(x^T_t; \mathbf{S}^I, s^T_{<t})} =  0.
\end{equation}
\end{theorem}

\begin{proof}
From Lemma~\ref{lem:entropy_growth}, the textual mutual information grows at least linearly with $t$. From Lemma~\ref{lem:visual_bound}, the visual mutual information is upper bounded by $C$. Since:

\[
\mathcal{I}(x^T_t; \mathbf{S}^I, s^T_{<t}) = \mathcal{I}(x^T_t; \mathbf{S}^I) + \mathcal{I}(x^T_t; s^T_{<t} \mid \mathbf{S}^I),
\]

and $\mathcal{I}(x^T_t; s^T_{<t} \mid \mathbf{S}^I)$ grows unboundedly with $t$, the ratio tends to zero, see Appendix \ref{app:proofs}.
\end{proof}

This theorem formalizes that the implicit alignment
 objective increasingly favors textual context over
  visual information as generation proceeds, 
  leading to weakened cross-modal alignment. 
  Detailed proofs of the above lemmas and theorem are provided in Appendix~\ref{app:proofs}.

\section{VISTA: Enhancing vision-text alignment}
\label{sec:method}

In this section, we introduce \method{} based on the theoretic analysis in the above section, an explicit alignment framework that directly optimizes visual-textual mutual information to prevent alignment degradation.

\subsection{Enhancing cross-modal alignment via explicit regularization}

Building upon the theoretical foundation established in Theorem~\ref{thm:vanishing_contribution}, we identify a fundamental limitation inherent to the implicit alignment objective: as the proportion of textual tokens grows during training, the strength of vision-text alignment progressively weakens, thereby impairing effective cross-modal integration. To mitigate this issue, we introduce an \emph{explicit alignment} regularizer that directly enforces stronger correspondence between visual and textual representations.

Formally, we augment the conventional cross-entropy loss $\mathcal{L}_{\text{CE}}$ with an explicit alignment term, yielding the composite objective:
\begin{equation}
\mathcal{L}_{\text{Total}} = \mathcal{L}_{\text{CE}} + \mathcal{L}_{\text{VISTA}} = \mathcal{L}_{\text{CE}} - \frac{1}{m} \sum_{t=1}^m f(t) \cdot \mathcal{I}(x_t; \mathbf{S}^I),
\label{eq:explicit_alignment_loss_rewrite_refined}
\end{equation}
where $\mathcal{I}(x_t; \mathbf{S}^I)$ denotes the mutual information between the $t$-th textual token embedding $x_t$ and the visual hidden states $\mathbf{S}^I$, and $f(t)$ is a monotonically increasing weighting function over token positions designed to counteract the natural attenuation of alignment signals for later tokens. A rigorous theoretical justification for the choice of $f(t)$ is provided in Appendix~\ref{appendix:ft_proof}.

\subsection{Practical instantiation of vision-text alignment}

To operationalize the explicit alignment term in our autoregressive decoder-only architecture, we leverage the hierarchical visual hidden states \(\mathbf{S}^I = \{S^I_1, \ldots, S^I_n\}\), where the final state \(S^I_n\) acts as a global summary of the visual input. The decoder-only design ensures the last hidden state summarizes all previous tokens, representing the entire input.

We approximate the mutual information $\mathcal{I}(x_t; \mathbf{S}^I)$ by a tractable distance metric between the textual embedding $x_t$ and the aggregated visual representation $S^I_n$. Concretely, we adopt the squared Euclidean distance as a surrogate measure:
\begin{equation}
\mathcal{L}_{\text{VISTA}} = \frac{1}{m} \sum_{t=1}^m f(t) \cdot \| x_t - S^I_n \|_2^2.
\label{eq:vista_l2_loss_refined}
\end{equation}

An initial attempt with a linear weighting $f(t) = t$ effectively emphasizes later tokens but causes the alignment loss to scale approximately as $\frac{m+1}{2}$ with sequence length $m$, potentially overwhelming the overall objective and destabilizing training dynamics.

To ensure stable optimization, we normalize the weighting by the sequence length, setting $f(t) = \frac{t}{m}$. This normalization bounds the magnitude of the alignment term, maintaining a consistent contribution regardless of $m$:
\begin{equation}
\mathcal{L}_{\text{Total}} = \mathcal{L}_{\text{CE}} + \frac{1}{m} \sum_{t=1}^m \frac{t}{m} \cdot \| x_t - S^I_n \|_2^2.
\label{eq:final_total_loss_vista_refined}
\end{equation}

This principled design preserves training stability while effectively reinforcing cross-modal integration. We term this approach \textbf{Vision-Text Alignment} (\method), which robustly mitigates the degradation of modal alignment and ensures sustained synergy between visual and textual modalities throughout training. The overall framework is depicted in Figure~\ref{VISTA}.

\section{Experiments}
\label{sec:experiment}

\subsection{Experimental setup}

\textbf{Datasets.}  
To ensure a fair and rigorous evaluation, we follow the dataset configuration used in LLaVA-v1.5. For pre-training, a subset of the LAION/CC/SBU dataset filtered for balanced concept coverage and augmented with BLIP-generated captions is employed, consistent with LLaVA's setup. Instruction tuning utilizes a combination of COCO~\cite{lin2014microsoft}, GQA~\cite{hudson2019gqa}, OCR-VQA~\cite{mishra2019ocr}, TextVQA~\cite{singh2019towards}, and VisualGenome~\cite{krishna2017visual}, as in~\cite{liu2024improved}. Further dataset details can be found therein.

\textbf{Tasks and evaluation.}  
Building upon prior benchmarks, we conduct comprehensive evaluations across three categories: (1) high-level semantic VQA tasks, encompassing VQAv2~\cite{vqav2}, OK-VQA~\cite{marino2019ok}, GQA~\cite{hudson2019gqa}, TextVQA~\cite{singh2019towards}, RealWorldQA~\cite{grok15v}, and DocVQA~\cite{docvqa}; (2) general multimodal understanding benchmarks covering diverse scenarios, including MMBench~\cite{liu2025mmbench}, SEED~\cite{li2023seed}, AI2D~\cite{kembhavi2016diagram}, MMMU~\cite{yue2024mmmu}, MMTB~\cite{ying2024mmt}, OCR-VQA~\cite{mishra2019ocr}, and MMStar~\cite{mmstar}; and (3) fine-grained visual perception and retrieval tasks, evaluated on MME~\cite{fu2024mmecomprehensiveevaluationbenchmark} alongside RefCOCO~\cite{refcoco}, RefCOCO+~\cite{refcoco}, and RefCOCOg~\cite{refcocog}. All evaluations are conducted using the lmms-eval framework~\cite{zhang2024lmmsevalrealitycheckevaluation, lmms_eval2024}, as recommended by LLaVA, which integrates multiple metrics to provide a unified and thorough assessment. Detailed dataset and metric descriptions are provided in Appendix~\ref{appendix:benchmark}.

\textbf{Models.}  
To rigorously validate \method{}, we experiment with two MLLM backbones of distinct scales: TinyLLaVA~\cite{zhou2024TinyLLava} (3B parameters) and LLaVA-v1.5~\cite{liu2024improved} (7B parameters). Both baselines are trained and fine-tuned on identical datasets to ensure comparability. For TinyLLaVA, we adopt the phi-2~\cite{phi2} 3B model as the language backbone and SigLIP~\cite{siglip} as the vision encoder. This setup enables a thorough assessment of \method{}'s generalizability and efficacy across varying model capacities and architectures.

\textbf{Baselines and implementation.}  
Following the data and training protocols in~\cite{liu2024improved}, we compare TinyLLaVA-\method{} with TinyLLaVA, and LLaVA-v1.5-\method{} with LLaVA-v1.5. The same weighting function $f(t)$ is used for both scales without hyperparameter tuning. Detailed model and training settings are provided in the appendix. During pre-training and fine-tuning, we optimize a combined loss of \method{} alignment and standard cross-entropy.

\begin{table}
\centering
\caption{Comparison of VQA results between the \method{} and LLava series models.}
\label{main_vqa}

\begin{center}

\begin{small}
\begin{sc}
\scalebox{0.9}{
\begin{tabular}{lccccccc}
\toprule

\textbf{METHOD} & \textbf{VQA\textsuperscript{v2}}&\textbf{VQA\textsuperscript{ok}} & \textbf{GQA}  & \textbf{VQA\textsuperscript{T}} & \textbf{RWQA} & \textbf{DOCVQA\textsuperscript{}} &\textbf{$\Delta$} \\
\midrule
\textbf{\emph{TinyLLava-3B-Sig-Phi-2 }} & & & & & & & \\ 
TinyLLava &79.13  & \multicolumn{1}{c}{57.50} & \multicolumn{1}{c}{61.20}  & \multicolumn{1}{c}{51.66} & \multicolumn{1}{c}{53.33} & \multicolumn{1}{c}{28.03} \\

\rowcolor[HTML]{ededed} 
TinyLLava-\method(cosine)& 
\text{79.10} & 
\multicolumn{1}{c}{\cellcolor{mygreen}\textbf{58.30}} &
\multicolumn{1}{c}{\cellcolor{mygreen}\text{  61.67 }} &  
\multicolumn{1}{c}{\cellcolor{mygreen}\textbf{51.89 }} &
\multicolumn{1}{c}{\cellcolor{mygreen}\text{ 53.59 }} & 
\multicolumn{1}{c}{\cellcolor{mygreen}\text{28.58 }} &
\cellcolor{white}\textbf{\textcolor{darkgreen}{$\uparrow$0.98\%}}\\
\rowcolor[HTML]{ededed} 
TinyLLava-\method&
\textbf{\cellcolor{mygreen}79.20 } & 
\multicolumn{1}{c}{\cellcolor{mygreen}\text{57.67}} & 
\multicolumn{1}{c}{\cellcolor{mygreen}\textbf{  61.79 }} & 
\multicolumn{1}{c}{\cellcolor{mygreen}\textbf{51.89 }} & 
\multicolumn{1}{c}{\cellcolor{mygreen}\textbf{ 55.42 }} & 
\multicolumn{1}{c}{\cellcolor{mygreen}\textbf{28.79 }} &
\cellcolor{white}\textbf{\textcolor{darkgreen}{$\uparrow$1.33\%}}
\\
\midrule

\textbf{\emph{LLava-v1.5-7B}}& & & & & & & \\ 
LLava& 78.50 & \multicolumn{1}{c}{53.44} & \multicolumn{1}{c}{62.00} & \multicolumn{1}{c}{46.07} & \multicolumn{1}{c}{55.82} & \multicolumn{1}{c}{21.49} \\

\rowcolor[HTML]{ededed} 
LLava-\method(cosine)& 
\cellcolor{mygreen}\text{78.98}& 
\cellcolor{mygreen}\textbf{56.30} &
\cellcolor{mygreen}\textbf{62.90} &
\cellcolor{mygreen}\textbf{46.77} &
\cellcolor{mygreen}\text{56.34} & 
\cellcolor{mygreen}\text{21.87} &
\cellcolor{white}\textbf{\textcolor{darkgreen}{$\uparrow$1.94\%}} \\
\rowcolor[HTML]{ededed} 
LLava-\method& 
\cellcolor{mygreen}\textbf{79.06}& 
\cellcolor{mygreen}\text{56.25} & 
\cellcolor{mygreen}\text{62.55} & 
\cellcolor{mygreen}\text{46.63} & 
\cellcolor{mygreen}\textbf{57.39} & 
\cellcolor{mygreen}\textbf{22.66} &
\cellcolor{white}\textbf{\textcolor{darkgreen}{$\uparrow$2.72\%}} \\
\bottomrule
\end{tabular}}
\end{sc}
\end{small}
\end{center}
\end{table}

\begin{table}
\centering
\caption{Comparison of the general visual task results between the \method{} and llava series models.}
\label{main_general}
\begin{center}
\begin{small}
\begin{sc}
\scalebox{0.75}{
\begin{tabular}{lccccccccccc@{}}
\toprule

\multirow{2}{*}{\textbf{METHOD}}  &
\multirow{2}{*}{\textbf{SEEDB\textsuperscript{I}}}& 
\multirow{2}{*}{\textbf{AI2D}} & 
\multirow{2}{*}{\textbf{MMMU\textsuperscript{}}}&
\multirow{2}{*}{\textbf{MMStar}} & 
\multirow{2}{*}{\textbf{MMTB}} & 
\multirow{2}{*}{\textbf{OCRB}} & 
\multicolumn{2}{c}{\textbf{MMBench}} &
\multirow{2}{*}{\textbf{$\Delta$}}\\ 
\cmidrule(lr){8-9}
& && & & & &    \textbf{en} & \textbf{cn} \\ 
\midrule
\textbf{\emph{TinyLLava-3B-Sig-Phi-2}} & && & & & & & & &\\ 
TinyLLava  & 
69.01&
60.36 &
36.33 &
37.19 &
48.73 &
337 &
67.04 & 
42.37&
\cellcolor{white}\text{} \\
\rowcolor
[HTML]{EDEDED}
TinyLLava-\method{} (cosine) & 
68.88& 
\text{59.55 } & 
\cellcolor{mygreen}\text{36.89} & 
\cellcolor{mygreen}\text{37.65 } & 
\cellcolor{mygreen}\text{48.77} &
334 & 
\cellcolor{mygreen}\textbf{68.10 }&
\cellcolor{mygreen}\text{42.60}  &
\cellcolor{white}\textbf{\textcolor{darkgreen}{$\uparrow$0.22\%}} \\

\rowcolor
[HTML]{EDEDED}
TinyLLava-\method{}  & 
\cellcolor{mygreen}\textbf{69.37}& 
\text{59.84 } & 
\cellcolor{mygreen}\textbf{37.89} & 
\cellcolor{mygreen}\textbf{39.91 } & 
\cellcolor{mygreen}\textbf{48.86} & 
\cellcolor{mygreen}\textbf{338} & 
\cellcolor{mygreen}\text{ 66.53}& 
\cellcolor{mygreen}\textbf{43.89}  &
\cellcolor{white}\textbf{\textcolor{darkgreen}{$\uparrow$1.82\%}} \\

\midrule
\textbf{\emph{LLava-v1.5-7B}}  & & & & & & & && \\ 
LLava &
66.17 &  
54.80 &
36.44 & 
33.48 &
48.86 &
313 & 
64.30 &
58.30 &
\cellcolor{white}\text{} \\

\rowcolor
[HTML]{EDEDED}
LLava-\method(cosine) & 
\cellcolor{mygreen}66.43&  
\cellcolor{mygreen}\text{56.28}   &
35.88 & 
\cellcolor{mygreen}\textbf{35.91}  & 
\cellcolor{mygreen}\text{48.06} &  
\cellcolor{mygreen}\text{315} & 
\cellcolor{mygreen}\textbf{66.08} &
57.85  &
\cellcolor{white}\textbf{\textcolor{darkgreen}{$\uparrow$1.27\%}} \\
\rowcolor
[HTML]{EDEDED}
LLava-\method{} &
\cellcolor{mygreen}66.68 &   
\cellcolor{mygreen}\textbf{56.70}   &  
\cellcolor{mygreen}\textbf{37.22}  & 
\cellcolor{mygreen}\text{35.88}  & 
\cellcolor{mygreen}\textbf{48.86}  & 
\cellcolor{mygreen}\textbf{321} &
\cellcolor{mygreen}\text{65.36} &
57.29 &
\cellcolor{white}\textbf{\textcolor{darkgreen}{$\uparrow$1.97\%}} \\
\bottomrule
\end{tabular}}
\end{sc}
\end{small}
\end{center}
\end{table}

\begin{table}[t]
\centering
\caption{Comparison of low level visual tasks results between the \method{} and llava series models.}
\label{main_low_level}
    \begin{center}
\begin{small}
\begin{sc}
\scalebox{0.75}{
\begin{tabular}{@{}p{4.3cm}ccccccc@{}}
\toprule
\multirow{2}{*}{\raggedright\textbf{METHOD}} & \multicolumn{2}{c}{\textbf{MME}} & \multicolumn{3}{c}{\textbf{RefCoco}} & \multirow{2}{*}{\textbf{$\Delta$}} \\
 \cmidrule(lr){2-3}  \cmidrule(lr){4-6} 
&\textbf{PERCEPTION}  & \textbf{COGNITION}  & \textbf{REFCOCO} & \textbf{REFCOCO+} & \textbf{REFCOCOG}  \\ 
\midrule

\textbf{\emph{TinyLLava-3B}} & & & & & &  \\ 
TinyLLava & 1457.09 & 329.64 & 28.67 & 28.39 & 59.54 & \\
\rowcolor[HTML]{ededed} 
TinyLLava-\method(Cosine)  & 
\cellcolor{mygreen}\text{1507.09} & 
\cellcolor{mygreen}\textbf{358.92} & 
\cellcolor{mygreen}\textbf{29.96} & 
\cellcolor{mygreen}\textbf{29.21} & 
\text{58.21} &  
\cellcolor{white}\textbf{$\Delta$\textcolor{darkgreen}{$\uparrow$2.55\%}} \\ 
\rowcolor[HTML]{ededed} 
TinyLLava-\method{}  & 
\cellcolor{mygreen}\textbf{1517.72} & 
\cellcolor{mygreen}\text{335.35} & 
\text{28.15} & 
\text{28.31} & 
\cellcolor{mygreen}\textbf{62.94} &  
\cellcolor{white}\textbf{$\Delta$\textcolor{darkgreen}{$\uparrow$3.03\%}} \\
\midrule
\textbf{\emph{LLava-v1.5-7B}} & & & & & &\\ 
LLava & 1510.72 & 348.20 & 29.76 & 28.92 & 57.76& \\
\rowcolor[HTML]{ededed} 
LLava-\method(cosine)  &
\cellcolor{mygreen}\textbf{1514.16} & 
\cellcolor{mygreen}\text{361.07} &
\cellcolor{mygreen}\textbf{31.81} &  
\cellcolor{mygreen}\text{31.24}& 
\text{56.37} & 
\cellcolor{white}\textbf{$\Delta$\textcolor{darkgreen}{$\uparrow$1.74\%}} \\
\rowcolor[HTML]{ededed} 
LLava-\method{} & 
\text{1506.70} & 
\cellcolor{mygreen}\textbf{377.85} & 
\cellcolor{mygreen}\text{32.17} &  
\cellcolor{mygreen}\textbf{31.54}& 
\text{57.47} & 
\cellcolor{white}\textbf{$\Delta$\textcolor{darkgreen}{$\uparrow$2.74\%}} \\

\bottomrule
\end{tabular}
}
\end{sc}
\end{small}
\end{center}
\end{table}

\subsection{Experimental results}

The results, summarized in Tables \ref{main_vqa}, \ref{main_general}, and \ref{main_low_level}, demonstrate that \method{} consistently enhances multimodal understanding across different model scales and task types without introducing additional trainable parameters or requiring extra data, aligning with our core contributions.

\paragraph{\method{} consistently improves performance on high-level semantic VQA tasks, validating its effectiveness in enhancing vision-text alignment for complex reasoning.}
As shown in Table~\ref{main_vqa}, integrating our proposed \method{} alignment objective into both TinyLLaVA-3B and LLaVA-v1.5-7B models yields significant performance gains over their original counterparts on high-level VQA datasets. For example, LLaVA-v1.5-\method{} achieves an average improvement of +2.72\% over its baseline, with strong individual gains observed on RealWorldQA (+1.57 points, from 55.82 to 57.39) and DocVQA (+1.17 points, from 21.49 to 22.66). Similarly, TinyLLaVA-\method{} shows an average improvement of +1.33\%. These improvements demonstrate that VISTA's explicit cross-modal mutual information maximization, as theorized, effectively strengthens the model's ability to integrate visual information for complex VQA tasks. This is achieved efficiently, without requiring additional trainable parameters, data, or architectural changes, directly aligning with our contribution of a lightweight yet powerful alignment method.

\paragraph{\method{} enhances general multimodal understanding across diverse reasoning tasks and model scales, demonstrating its broad applicability and scalability.}
In Table~\ref{main_general}, VISTA consistently improves performance on general vision-language understanding tasks. The TinyLLaVA-\method{} model achieves an average improvement of +1.82\% across these tasks, with a notable +2.72 point increase (a +7.31\% relative gain) on the challenging MMStar benchmark (from 37.19 to 39.91). Similarly, LLaVA-v1.5-\method{} shows an average improvement of +1.97\%, also with a significant gain on MMStar by +2.40 points (a +7.17\% relative gain, from 33.48 to 35.88) and on AI2D by +1.90 points (from 54.80 to 56.70). These results confirm that the explicit alignment fostered by VISTA generalizes effectively across various domains—from complex benchmarks like MMStar to diagram understanding in AI2D—and scales robustly from the 3B TinyLLaVA to the 7B LLaVA-v1.5. This substantiates VISTA's utility as a scalable and architecture-agnostic enhancement for multimodal alignment, aligning with our third key contribution.

\paragraph{\method{} leads to substantial improvements in fine-grained visual perception and comprehension, particularly in cognitively demanding tasks.}
Table~\ref{main_low_level} demonstrates \method{}'s effectiveness in fine-grained visual tasks assessed by MME and referring expression grounding. For LLaVA-v1.5-\method{}, the MME Cognition subscore notably increases from 348.20 to 377.85 (+29.65 points, an +8.52\% relative gain), highlighting \method{}'s significant contribution to complex visual understanding and reasoning, as claimed in the abstract. Similarly, TinyLLaVA-\method{} improves on the RefCOCOg benchmark by +3.4 points (from 59.54 to 62.94), showing that the proposed alignment objective also reinforces spatial and referential grounding abilities. These gains clearly indicate that \method{} not only addresses high-level reasoning challenges but also enhances models' sensitivity to localized visual semantics, aligning with the goal of improving overall visual understanding by preventing the degradation of visual alignment.

\paragraph{The mutual information-based alignment in \method{} is consistently beneficial across both lightweight and large MLLM architectures.}  
Across all tasks—high-level VQA, general reasoning, and fine-grained perception—\method{} delivers consistent performance upgrades from TinyLLaVA (3B) to LLaVA-v1.5 (7B), with average increases ranging from +1.3\% to +2.7\%. Importantly, both cosine contrastive losses and mutual information-based variants lead to improvements, with the latter generally achieving higher gains. Notably, these enhancements are achieved without architecture modifications, additional trainable parameters, or extra training stages. This efficiency aligns with the paper's highlighted contribution: \method{} is a principled, lightweight, and practically deployable strategy for improving cross-modal alignment.

\paragraph{The choice of mutual information estimator impacts VISTA's efficacy, with L2 distance outperforming cosine similarity due to more aligned optimization dynamics.}
We also conducted an ablation study using the cosine similarity function to estimate $\mathcal{I}(x_t^T, S^I_n)$, with the results summarized in Tables \ref{main_vqa}--\ref{main_low_level} (rows labeled ``\method{}(cosine)''). 
The results indicate that the cosine-based estimator still improves the 3B/7B backbones by +1.25\% / +1.65\% on average, yet consistently trails the L2 version.  The reason is that the cosine loss grows in magnitude (negative sign) as alignment improves, potentially conflicting with the decreasing cross-entropy objective, whereas the L2 distance decreases monotonically in tandem.  This analysis offers a concrete design guideline for future MI-based alignment methods.

\paragraph{Observed performance variations on specific benchmarks highlight \method{}'s primary strength in visual-textual alignment.}
\method{} primarily excels at aligning visual and textual information. While generally beneficial, it sometimes causes slight performance drops on benchmarks like AI2D and MMBench (cn), which depend more on text understanding than visual details. This suggests \method{} boosts visual-text feature alignment rather than just text skills, a point reinforced by strong results on vision-heavy benchmarks like MMStar.

On RefCOCO benchmarks, performance varies. For instance, ``TinyLLaVA-\method{}'' improved with longer text queries (RefCOCOG) but dipped with shorter ones (RefCOCO/RefCOCO+), while ``LLaVA-\method{}'' showed the reverse. This is likely because \method{} doesn't specifically train to balance performance across different text query lengths, and models might adaptively focus on easier improvements based on their size. A similar varied pattern on RefCOCO with ``TinyLLaVA-\method{}(cosine)'' supports this.

\subsection{Case study}

\begin{figure*}[ht]
\begin{center}
\begin{tabular}{cc}
\includegraphics[width=0.42\textwidth]{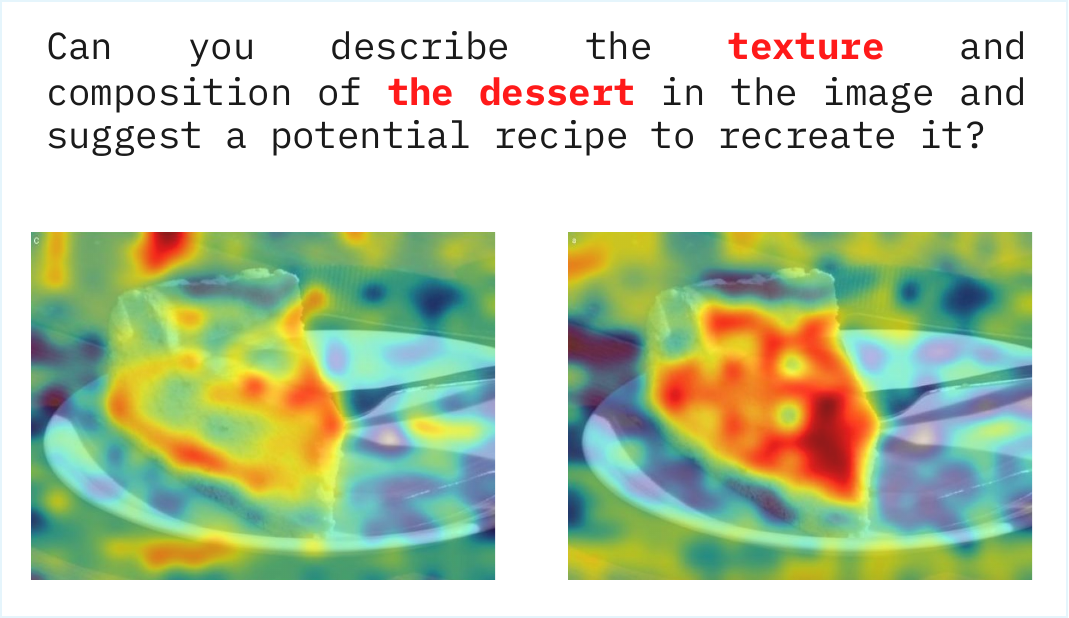} & 
\includegraphics[width=0.45\textwidth]{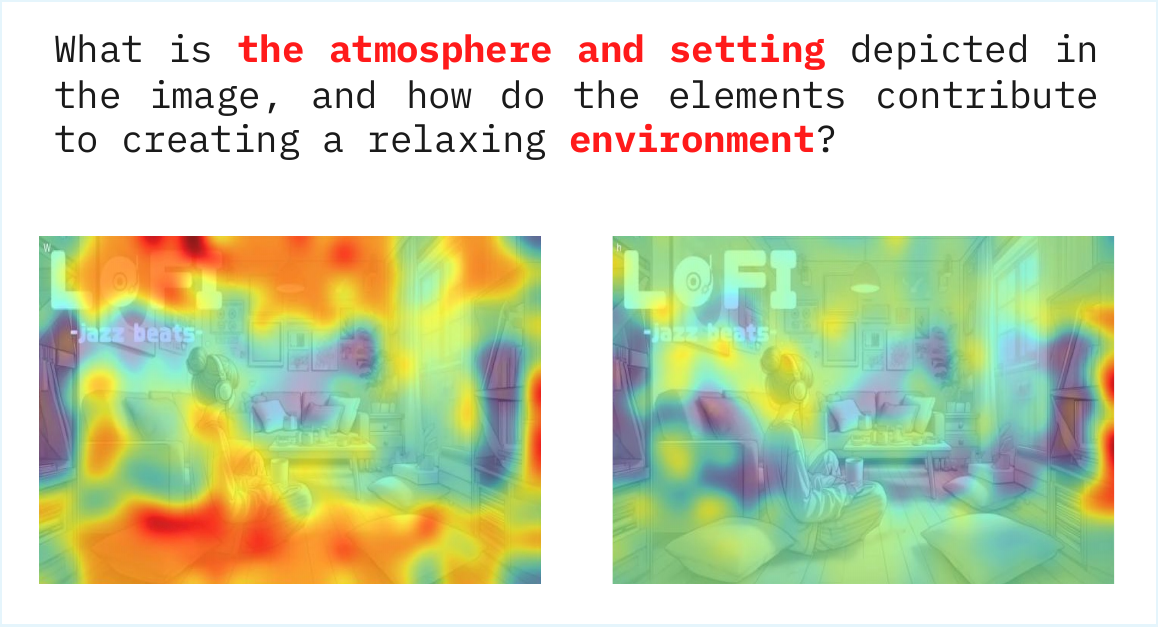} \\[6pt]
(a) Object  & (b) Enviroment \\[12pt]
\includegraphics[width=0.42\textwidth]{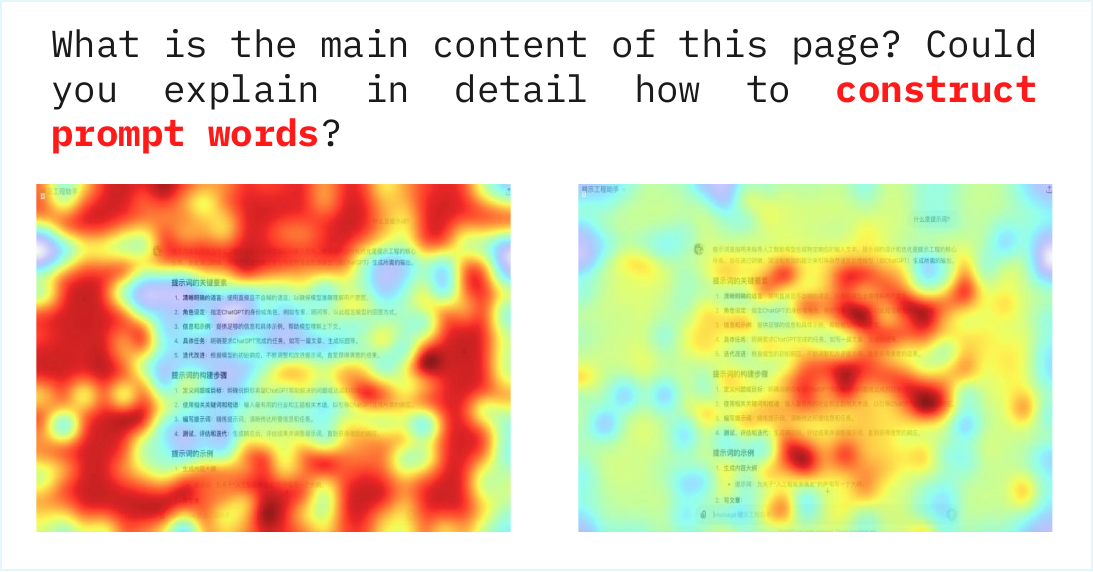} & 
\includegraphics[width=0.45 \textwidth]{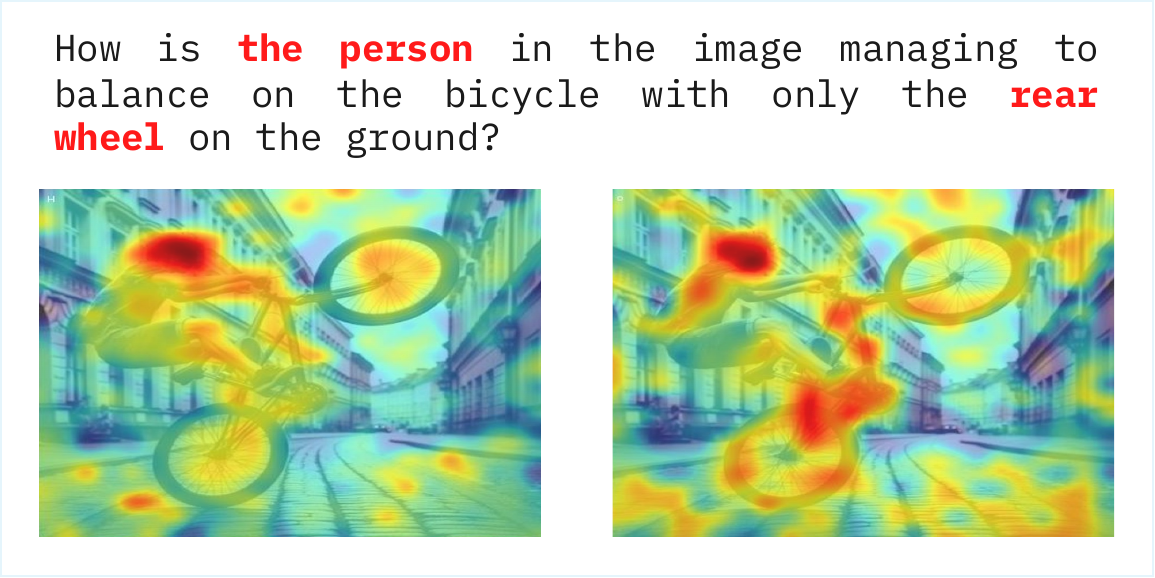} \\[6pt]
(c) Text & (d) Character
\end{tabular}
\caption{Semantic similarity visualization between MLP-generated prompt embeddings and image embeddings on LLaVA Wild Bench\cite{llava}. Each subfigure presents two images per case: left—LLaVA baseline; right—our \method{} method. Darker colors indicate higher semantic alignment.}
\label{fig:cases}
\end{center}
\end{figure*}

We present several illustrative cases in Figure~\ref{fig:cases} to highlight the qualitative advantages of our \method{} method over the LLaVA baseline in semantic alignment. Without relying on extra annotations, \method{} naturally learns more precise and coherent cross-modal correspondences across diverse scenarios, including objects, environments, text, and characters.

Notably, \method{} demonstrates more focused alignment on entire objects (a) and produces spatially balanced semantic maps for environments (b), closely matching human perception.  In the text case (c), unlike LLaVA's diffuse background alignment, \method{} precisely localizes semantics to the relevant text regions.  Similarly, in (d), it distinctly associates both the person and bicycle mentioned in the prompt, whereas LLaVA's alignment is more scattered.  These examples collectively validate \method's superior capability to achieve fine-grained and semantically meaningful vision-language alignment.


\section{Conclusion}


In this work, we addressed the critical challenge of suboptimal vision-language alignment in MLLMs, which often stems from the limitations of conventional text-based cross-entropy loss. Our information-theoretic analysis revealed that the implicit alignment fostered by this loss degrades as text length increases, diminishing the effective integration of visual information. In response to this problem, we propose a novel method \method{} that introduces an explicit objective to maximize cross-modal mutual information. A key advantage of \method{} is its simplicity and efficiency, enhancing alignment without necessitating additional parameters, architectural changes, or supplementary training data.

Our extensive experimental results across a wide array of benchmarks robustly demonstrate \method's effectiveness, yielding significant performance improvements in visual understanding and reasoning tasks. These findings underscore the benefits of explicit alignment objectives and suggest that \method{} offers a practical and scalable approach to building more capable and balanced MLLMs. We believe this work not only provides a valuable tool for current MLLMs development but also opens promising avenues for future research into principled and efficient cross-modal alignment techniques.

\newpage
\bibliographystyle{plainnat}
\bibliography{main.bbl}

\appendix

\section{Detailed Proofs}
\label{app:proofs}

\subsection{Proof of Lemma~\ref{lem:entropy_growth} (Autoregressive Entropy Growth)}

\begin{proof}
We consider the entropy of the text prefix up to step \(t-1\), denoted as \(x^T_{<t} = (x^T_1, x^T_2, \ldots, x^T_{t-1})\). By the chain rule of entropy, we have

\[
\mathcal{H}(x^T_{<t}) = \sum_{k=1}^{t-1} \mathcal{H}(x^T_k \mid x^T_{<k}),
\]

where \(x^T_{<k} = (x^T_1, \ldots, x^T_{k-1})\).

Under the non-degeneracy assumption on the data distribution, each conditional entropy term is strictly positive and uniformly lower bounded by some constant \(\Delta_H > 0\), i.e.,

\[
\mathcal{H}(x^T_k \mid x^T_{<k}) \geq \Delta_H, \quad \forall k \in \{1, \ldots, t-1\}.
\]

This assumption excludes trivial or deterministic sequences where the next token is fully determined by the prefix.

Combining these, we obtain

\[
\mathcal{H}(x^T_{<t}) = \sum_{k=1}^{t-1} \mathcal{H}(x^T_k \mid x^T_{<k}) \geq \sum_{k=1}^{t-1} \Delta_H = (t-1) \Delta_H.
\]

As \(t\) grows large, the difference between \(t\) and \(t-1\) becomes negligible, so we can write

\[
\mathcal{H}(x^T_{<t}) \geq (t-1) \Delta_H \approx t \Delta_H,
\]

which establishes that the entropy of the text prefix grows at least linearly with the sequence length \(t\).

\end{proof}

\subsection{Proof of Lemma~\ref{lem:visual_bound} (Bounded Visual Mutual Information)}

\begin{proof}
By the definition of mutual information,

\[
\mathcal{I}(x^T_t; \mathbf{S}^I) = \mathcal{H}(x^T_t) - \mathcal{H}(x^T_t \mid \mathbf{S}^I).
\]

Since conditional entropy is non-negative, it follows that

\[
\mathcal{I}(x^T_t; \mathbf{S}^I) \leq \mathcal{H}(x^T_t).
\]

Because \(x^T_t\) is a discrete token drawn from a finite vocabulary \(\mathcal{V}\), its entropy is upper bounded by the logarithm of the vocabulary size:

\[
\mathcal{H}(x^T_t) \leq \log |\mathcal{V}| = C,
\]

where \(C\) is a constant independent of the token position \(t\).

Therefore,

\[
\mathcal{I}(x^T_t; \mathbf{S}^I) \leq C,
\]

which means the mutual information between the visual signals and any single token is uniformly bounded by a constant.

\end{proof}

\subsection{Proof of Theorem~\ref{thm:vanishing_contribution} (Vanishing Visual Alignment Contribution)}

\begin{proof}
We analyze the mutual information between the current text token \(x^T_t\) and the combined modalities: the visual signals \(\mathbf{S}^I\) and the preceding text tokens \(s^T_{<t}\).

By the chain rule of mutual information, we have

\[
\mathcal{I}(x^T_t; \mathbf{S}^I, s^T_{<t}) = \mathcal{I}(x^T_t; \mathbf{S}^I) + \mathcal{I}(x^T_t; s^T_{<t} \mid \mathbf{S}^I).
\]

First, from Lemma~\ref{lem:visual_bound}, the mutual information between the visual signals and any single token is upper bounded by a constant:

\[
\mathcal{I}(x^T_t; \mathbf{S}^I) \leq C,
\]

where \(C = \log |\mathcal{V}|\) is the logarithm of the vocabulary size, independent of the time step \(t\).

Next, consider the conditional mutual information term \(\mathcal{I}(x^T_t; s^T_{<t} \mid \mathbf{S}^I)\). Using the non-negativity of mutual information and the chain rule, we have

\[
\mathcal{I}(x^T_t; s^T_{<t} \mid \mathbf{S}^I) = \mathcal{I}(x^T_t; \mathbf{S}^I, s^T_{<t}) - \mathcal{I}(x^T_t; \mathbf{S}^I) \geq \mathcal{I}(x^T_t; s^T_{<t}) - \mathcal{I}(x^T_t; \mathbf{S}^I).
\]

Here, we used the monotonicity of mutual information:

\[
\mathcal{I}(x^T_t; s^T_{<t}) \leq \mathcal{I}(x^T_t; \mathbf{S}^I, s^T_{<t}).
\]

From Lemma~\ref{lem:entropy_growth}, the mutual information between the current token and the preceding text context grows at least linearly with \(t\):

\[
\mathcal{I}(x^T_t; s^T_{<t}) = \mathcal{H}(x^T_t) - \mathcal{H}(x^T_t \mid s^T_{<t}) \geq t \Delta_H - \epsilon,
\]

where \(\Delta_H > 0\) is a positive lower bound on the conditional entropy, and \(\epsilon > 0\) is a small constant accounting for finite-sample or boundary effects.

Combining the above inequalities yields

\[
\mathcal{I}(x^T_t; s^T_{<t} \mid \mathbf{S}^I) \geq t \Delta_H - \epsilon - C.
\]

Therefore, the total mutual information satisfies

\[
\mathcal{I}(x^T_t; \mathbf{S}^I, s^T_{<t}) = \mathcal{I}(x^T_t; \mathbf{S}^I) + \mathcal{I}(x^T_t; s^T_{<t} \mid \mathbf{S}^I) \geq 0 + (t \Delta_H - \epsilon - C) = t \Delta_H - \epsilon - C.
\]

As \(t \to \infty\), the linear term \(t \Delta_H\) dominates, so the total mutual information grows at least linearly with the context length.

Finally, consider the ratio of the visual mutual information to the total mutual information:

\[
\frac{\mathcal{I}(x^T_t; \mathbf{S}^I)}{\mathcal{I}(x^T_t; \mathbf{S}^I, s^T_{<t})} \leq \frac{C}{t \Delta_H - \epsilon - C}.
\]

Taking the limit as \(t \to \infty\), the denominator diverges to infinity, and thus the ratio converges to zero:

\[
\lim_{t \to \infty} \frac{\mathcal{I}(x^T_t; \mathbf{S}^I)}{\mathcal{I}(x^T_t; \mathbf{S}^I, s^T_{<t})} = 0.
\]

This rigorously establishes that the relative contribution of the visual modality to the information about the current token vanishes as the text context length grows without bound.

\end{proof}
\newpage
\section{Theoretical Justification for the Weighting Function $f(t)$}
\label{appendix:ft_proof}

In this appendix, we provide a rigorous derivation explaining why the weighting function $f(t)$ should increase with token position $t$, and how this relates to an adaptive parameter $\lambda$ that controls the balance between visual and textual mutual information terms.

\subsection{Setup and Definitions}

When we want to introduce a display alignment target, assuming the weight of this target is $\lambda$, then the current loss can be written as:
\[
\mathcal{L} = \mathcal{L}_{\text{CE}} + \lambda \mathcal{L}_{\text{VISTA}}
 = \mathcal{L}_{\text{CE}} - \lambda \frac{1}{m} \sum_{t=1}^{m} \mathcal{I}(x_t; \mathbf{S}^I)
\]
We set the visual proportion $p(t)$ 
as the ratio of the visual mutual information 
to the total mutual information when at token position $t$:
\[
p(t) = \frac{\mathcal{I}(x_t; \mathbf{S}^I)}{\mathcal{I}(x_t; \mathbf{S}^I) + \mathcal{I}(x_t; s_{<t}^T | \mathbf{S}^I)}
\]

Because it is to maximize mutual information, 
the alignment target is displayed. 
The target is a negative sign in the loss function. 
The enhanced visual proportion  at token position $t$ becomes:
\[
\rho_{\text{VISTA}}(t) = \frac{(1+\lambda) \mathcal{I}(x_t; \mathbf{S}^I)}{(1+\lambda) \mathcal{I}(x_t; \mathbf{S}^I) + \mathcal{I}(x_t; s_{<t}^T | \mathbf{S}^I)},
\]

where $\lambda > 0$ is a balancing parameter introduced to explicitly weight the visual mutual information term.

Rearranging the equation of $\rho_{\text{VISTA}}(t)$, we can express $\lambda$ as a function of $\rho_{\text{VSITA}}(t)$ and the mutual information terms:
\[
\lambda = \frac{\mathcal{I}(x_t; s_{<t}^T | \mathbf{S}^I)}{\left(\frac{1}{\rho_{\text{VISTA}}(t)} - 1\right) \mathcal{I}(x_t; \mathbf{S}^I)} - 1.
\]

\subsection{Interpreting $\lambda$ as a Function of $f(t)$}

To prevent the modal alignment ratio $\rho_{\text{VISTA}}(t)$ from vanishing as $t$ grows, we require $\lambda$ to increase with $t$. Intuitively, this means placing more emphasis on the visual mutual information term for later tokens, which are more prone to losing cross-modal alignment.

We propose to set $\lambda$ proportional to the weighting function $f(t)$:
\[
\lambda \sim f(t).
\]

Empirically and theoretically, choosing $f(t)$ as a linear function, e.g., $f(t) = t$, effectively slows the decay of $\rho_{\text{VISTA}}(t)$ from an exponential to a linear rate $\mathcal{O}(k)$, where $k$ is a constant.

\subsection{Lower Bound on $\rho_{\text{VISTA}}(t)$}

Substituting $\lambda = f(t)$ into the equation of $\rho_{\text{VSITA}}(t)$, we obtain:

\[
\rho_{\text{VISTA}}(t) = \frac{1}{1 + \frac{\mathcal{I}(x_t; s_{<t}^T | \mathbf{S}^I)}{(1 + f(t)) \mathcal{I}(x_t; \mathbf{S}^I)}} \geq \frac{1}{1 + \frac{\mathcal{I}(x_t; s_{<t}^T | \mathbf{S}^I)}{f(t) \mathcal{I}(x_t; \mathbf{S}^I)}}.
\]

As $t \to \infty$, if $f(t)$ grows linearly, the denominator stabilizes, and $\rho_{\text{VISTA}}(t)$ is bounded below by a positive constant:

\[
\lim_{t \to \infty} \rho_{\text{VSITA}}(t) \geq \frac{1}{1 + \Delta_H},
\]

where $\Delta_H$ is a finite constant representing the asymptotic ratio of conditional to visual mutual information.
This ensures that $\rho_{\text{VISTA}}(t)$ does not approach 0.
\subsection{Implications for the Weighting Function}

This analysis shows that by choosing $f(t)$ to grow linearly with $t$, we effectively prevent the modal alignment ratio from vanishing, maintaining a stable and meaningful alignment signal throughout the sequence.

Moreover, normalizing $f(t)$ by the sequence length $m$, i.e., $f(t) = \frac{t}{m}$, ensures that the overall alignment loss remains scale-invariant and does not overwhelm the primary cross-entropy loss during training.

\subsection{Summary}

In summary, the parameter $\lambda$ controlling the balance between visual and textual mutual information terms should increase with token position $t$. This naturally leads to a weighting function $f(t)$ that grows linearly with $t$, normalized by sequence length to maintain stable optimization. This theoretical insight justifies our final choice of
\[
f(t) = \frac{t}{m},
\]
which balances effective explicit alignment with training stability in the \method{} framework.


\newpage

\section{BenchMarks.}
\label{appendix:benchmark}
\subsection{Visual question answering benchmarks}
We conduct experiments on visual question-answering benchmarks, including, VQAV2, OK-VQA, GQA,  TextVQA, RealWorldQA, and DocVQA. 
OK-VQA includes questions that necessitate external knowledge beyond the multimodal inputs provided. 
GQA is specifically designed to assess the reasoning capabilities of the model. 
VQAV2 is one of the most widely used VQA evaluation sets. It covers a wide variety of visual question-answering tasks, and the number of test sets is huge enough to evaluate the visual capabilities of the model very well.
TextVQA places a greater emphasis on evaluating the model's ability to comprehend text within natural scenes. 
RealWorldQA is a benchmark specifically designed to evaluate the spatial understanding capabilities of multimodal AI models in real-world contexts. 
DocVQA is used for visual question answering (VQA) based on document images. The main focus is on evaluating the model's ability to understand the document structure.

These datasets are strategically selected to evaluate our method's capacity to understand comprehensively and reason across diverse visual contexts and knowledge domains.

\textbf{OK-VQA:} OK-VQA(Outside Knowledge VQA)\cite{marino2019ok} is a visual question answering dataset that requires external knowledge. The answers to the questions cannot be inferred solely from the image but also need to incorporate common sense or world knowledge. This dataset evaluates the model's ability in the intersection of vision and knowledge reasoning. We used the acc of  validation set to report Our result.

\textbf{GQA:} GQA(Graph Question Answering)\cite{hudson2019gqa} generates questions and answers based on image scene graphs, focusing on structured reasoning. It emphasizes logical analysis and challenges the model's depth of understanding of semantics and context.

\textbf{VQAV2:} VQAv2\cite{vqav2} is a new dataset containing open-ended questions about images. These questions require an understanding of vision, language and commonsense knowledge to answer. We used the test split to report Our result.

\textbf{TextVQA:} TextVQA\cite{singh2019towards} focuses on textual information in images, requiring models to recognize and comprehend text within images to answer questions. It drives research on the integration of visual and textual information, expanding the boundaries of visual question answering.

\textbf{RealWorldQA:} RealWorldQA\cite{grok15v} features images and questions sourced from real-world scenarios, encompassing diverse content from daily life. The dataset imposes higher requirements on the model's generalization ability and adaptability to complex scenes.

\textbf{DOCVQA:} This dataset\cite{docvqa} comprises over 50,000 questions formulated on more than 12,000 document images. Document Visual Question Answering (DocVQA) aims to promote a "purpose-driven" perspective within the field of Document Analysis and Recognition, encouraging research that directly addresses practical and semantic understanding of document content.

\subsection{General multimodal benchmarks}
We evaluate our proposed method on general multimodal benchmarks, including MMStar, MMBench, SEED-Bench,  AI2D, MMStar, MMMU, OCR-Bench and MMT-Bench. 

MMStar primarily targets evaluation tasks with a strong reliance on visual information. Candidate samples are initially filtered from existing benchmarks via an automated pipeline, followed by manual verification to ensure that each selected instance exhibits clear visual dependency, minimal data leakage, and requires advanced multimodal reasoning capabilities.
SEED-Bench focuses on assessing generative comprehension in multimodal large language models. POPE evaluates the extent of multimodal hallucinations present in a model. 
AI2D assesses a model's ability to interpret scientific diagram inputs. MM-Vet evaluates the multimodal conversational skills of a model using GPT-4 as a benchmark. 
MMMU is designed to assess multimodal models on extensive multi-disciplinary tasks that require college-level subject knowledge and deliberate reasoning. 
MMT-Bench is a comprehensive benchmark developed to evaluate MLLMs across a wide range of multimodal tasks requiring expert knowledge and deliberate visual recognition, localization, reasoning, and planning.
OCR-Bench contains a collection of 1,000 manually filtered and corrected question-answer pairs, covering five representative text-related tasks.
These diverse benchmarks provide a comprehensive framework for evaluating the performance and capabilities of our proposed method in multimodal learning.

\textbf{MMBench} MMBench(Multimodal Benchmark)\cite{liu2025mmbench} is a task-driven benchmark that focuses on systematically evaluating multimodal models across diverse real-world application scenarios, such as visual question answering, image captioning, and video understanding. Its emphasis on practical use cases highlights its importance for assessing the practical utility of MLLMs.

\textbf{SEED:} SEED(Spatial and Entity-aware Evaluation Dataset)\cite{li2023seed} is a benchmark specifically designed to evaluate the spatial and entity reasoning capabilities of multimodal models. By incorporating complex spatial relationships and entity-based queries, SEED tests a model's ability to perform fine-grained reasoning, which is critical for tasks such as scene understanding and object-oriented question answering.

\textbf{AI2D:} AI2D(Allen Institute for AI Diagram Dataset)\cite{kembhavi2016diagram} is a dataset centered on diagram understanding, designed to evaluate models' abilities to process non-photographic visual content. It focuses on reasoning over diagrams and charts, making it vital for tasks requiring scientific and technical visual comprehension.

\textbf{MMMU:} MMMU(Multimodal Multitasking Understanding)\cite{yue2024mmmu} evaluates the multitasking capabilities of multimodal models by testing their performance on multiple simultaneous tasks across different modalities. This benchmark is essential for assessing the adaptability and efficiency of models in dynamic, multitask scenarios.

\textbf{MMTB:} MMTB(Multimodal Task Benchmark)\cite{ying2024mmt} is a broad benchmark designed to evaluate the performance of multimodal models on a wide range of tasks, including vision-and-language navigation, multimodal reasoning, and image captioning. Its diversity makes it a strong indicator of a model's overall multimodal proficiency.

\textbf{OCRB:} OCRB (Optical Character Recognition Benchmark)\cite{mishra2019ocr} is a specialized benchmark for assessing a model's ability to recognize and interpret text in images. It focuses on OCR-related tasks, such as text detection, transcription, and contextual understanding, which are crucial for applications like document analysis and scene-text understanding.

\textbf{MMStar:} MMStar\cite{mmstar}, an elite vision-indispensable multi-modal benchmark comprising 1,500 samples meticulously selected by humans. MMStar benchmarks 6 core capabilities and 18 detailed axes, aiming to evaluate LVLMs' multi-modal capacities with carefully balanced and purified samples.



\subsection{Fine-grained Visual Perception and Referring Expression Comprehension Benchmarks}

To rigorously evaluate our model's capabilities in fine-grained visual perception and multimodal retrieval, we conduct experiments on several established benchmarks, including MME, RefCOCO, RefCOCO+, and RefCOCOg.

\textbf{MME:} \cite{fu2024mmecomprehensiveevaluationbenchmark} (Multimodal Evaluation) is a comprehensive benchmark designed to assess a model's proficiency in understanding and reasoning across multiple modalities, including vision, language, and audio. Comprising 14 diverse subtasks, MME evaluates both perceptual and cognitive aspects of multimodal intelligence, providing a standardized framework to measure cross-modal reasoning and generalization capabilities of multimodal large language models (MLLMs).

\textbf{RefCOCO:} \cite{refcoco} is a widely adopted referring expression comprehension dataset that challenges models to localize objects within images based on natural language descriptions. It emphasizes the alignment between fine-grained visual features and linguistic cues in everyday scenes, serving as a fundamental testbed for grounding language in vision.

\textbf{RefCOCO+:} \cite{refcoco} builds upon RefCOCO by excluding absolute spatial references (e.g., "left", "right") from the expressions, thereby compelling models to rely predominantly on appearance-based attributes for object localization. This constraint elevates the difficulty of the task and probes the model's ability to discern subtle visual details.

\textbf{RefCOCO+:} \cite{refcoco} builds upon RefCOCO by excluding absolute spatial references (e.g., "left", "right") from the expressions, thereby compelling models to rely predominantly on appearance-based attributes for object localization. This constraint elevates the difficulty of the task and probes the model's ability to discern subtle visual details.

\textbf{RefCOCOg:} \cite{refcocog} extends the complexity of referring expressions by providing longer, more descriptive, and contextually rich language annotations. This dataset evaluates a model's capacity to integrate detailed linguistic information with visual content, reflecting more naturalistic and challenging scenarios for referring expression comprehension.

Collectively, these benchmarks offer a comprehensive evaluation suite that spans from fundamental visual perception to sophisticated multimodal reasoning and retrieval, enabling a thorough assessment of a model's fine-grained understanding and cross-modal alignment capabilities.

\subsection{Lmms Eval}
We utilize lmms eval (v0.3.0)~\cite{lmms_eval2024} as the primary evaluation framework for our benchmark. lmms eval is a highly extensible and modular platform designed to accommodate a wide spectrum of evaluation tasks for multimodal large models. At its core, the framework leverages a declarative YAML configuration format, enabling precise specification and reproducibility of evaluation protocols. This design not only facilitates straightforward task definition—such as swapping dataset paths to evaluate on new benchmarks—but also supports sophisticated customization of evaluation logic, empowering researchers to tailor assessments to nuanced experimental setups. By coupling configuration files with explicit codebase commit hashes, lmms eval ensures that evaluation environments can be reliably shared and replicated across the community, thereby promoting transparency and rigor in model assessment. This combination of flexibility, reproducibility, and ease of use makes lmms eval an indispensable tool for advancing standardized evaluation in multimodal research.

\subsection{Metrics}
We adopt a comprehensive and task-specific evaluation protocol to rigorously assess our model's performance across a diverse set of multimodal benchmarks. For most visual question answering datasets—including AI2D, GQA, MME, MMMU, MMStar, OCRB, OK-VQA, RealWorldQA, and the validation split of TextVQA—we report standard accuracy metrics, which directly measure the proportion of correctly predicted answers and serve as a fundamental indicator of model proficiency. Recognizing the limitations of exact-match metrics in open-ended generation tasks, we employ GPT-4 Turbo-based evaluation for VQAv2 and MMBench, leveraging the semantic understanding and contextual reasoning capabilities of large language models to provide a more nuanced and human-aligned assessment. For document-centric question answering on DocVQA, we utilize the Average Normalized Levenshtein Similarity (ANLS) metric on the validation set, which effectively captures textual similarity while accommodating OCR-related noise. Referring expression comprehension benchmarks—RefCOCO, RefCOCO+, and RefCOCOg—are evaluated using the CIDEr metric, emphasizing the quality of language grounding and the alignment between generated expressions and visual targets. Additionally, SEED-Bench reports image-level scores that probe spatial and entity-aware reasoning, reflecting the model's capacity for fine-grained visual relational understanding. This diverse suite of metrics ensures a holistic evaluation of our model's capabilities, spanning from precise answer accuracy to sophisticated semantic alignment and multimodal reasoning.

\newpage
\section{Experiments Detail}
\label{appendix:experiments}
We adopt hyperparameter configurations consistent with those reported in TinyLLaVA and LLaVA to ensure comparability. Notably, our TinyLLaVA~\cite{zhou2024TinyLLava} baseline results are derived from our own experimental runs rather than relying on official benchmarks. Although TinyLLaVA~ originally utilizes the ShareGPT4V dataset, we exclusively employ the LLaVA\cite{liu2024improved} dataset in all experiments. This deliberate choice is driven by the need to maintain dataset uniformity, thereby enabling a controlled investigation of model scaling effects on performance without confounding variables introduced by differing data distributions.

Training is conducted on a cluster of eight NVIDIA A100 GPUs (40GB each) using DeepSpeed to facilitate efficient distributed optimization. We follow a two-stage training regimen: an initial pretraining phase with a learning rate of 1e-3, scheduled via a cosine decay with a 3\% warmup over one epoch, and a subsequent fine-tuning phase with a reduced learning rate of 2e-5. The pretraining stage employs a per-device batch size of 32 and gradient accumulation steps of 1, yielding an effective batch size of 256. Mixed precision (bf16) and gradient checkpointing are enabled to optimize memory footprint and training stability. During fine-tuning, the batch size is halved to 16 to accommodate increased model complexity and dataset scale. Checkpoints are saved every 500 steps, retaining a maximum of two to balance storage constraints and recovery flexibility. Weight decay is set to zero, and the maximum sequence length is fixed at 2048 tokens throughout training. Data loading is parallelized with four worker threads, and lazy preprocessing is employed to optimize I/O throughput. This comprehensive setup ensures efficient, stable, and reproducible training conducive to rigorous empirical evaluation.

\newpage
\section{Limitation }
\label{appendix:limitation}
While \method{} offers a parameter-free and tuning-free module that is both simple and broadly applicable, there remain several limitations worth noting. First, the theoretically derived weighting function \( f(t) \) employed in our method, although grounded in information-theoretic principles, is not guaranteed to be optimal in practice. Its form is based on certain assumptions and approximations that may not fully capture the complex dynamics of vision-language alignment across diverse datasets and model architectures. Second, our estimation of mutual information, which underpins the explicit alignment objective, relies on tractable approximations rather than exact computation. Consequently, the mutual information estimates may be biased or suboptimal, potentially limiting the full potential of the proposed approach. Despite these limitations, the simplicity and generality of \method{} make it a strong baseline for explicit cross-modal alignment. We anticipate that future work will explore more refined and adaptive formulations of the weighting function and improved mutual information estimators, thereby enhancing the effectiveness and extensibility of \method{} across a wider range of multimodal model training scenarios.

Additionally, due to the limited availability of open-source pretraining data for many leading multimodal large language models (MLLMs), our experimental validation focuses on TinyLLaVA and LLava-v1.5.  As such, further evaluation on a broader range of models and training settings would be valuable to fully assess the generalizability of \method{}.

Despite these considerations, the conceptual simplicity and general applicability of \method{} make it a promising baseline for explicit cross-modal alignment.  We look forward to future work exploring refined weighting strategies and improved mutual information estimation techniques, which may further enhance the effectiveness and adaptability of \method{} in diverse multimodal training scenarios

\newpage

\section{Broader Impacts}
\label{appendix:impact}
Our work on \method{} contributes to the development of more capable and balanced multi-modal large language models, which have the potential to significantly enhance a wide range of applications including assistive technologies, education, and human-computer interaction. By improving vision-language alignment in a principled and efficient manner, \method{} can help create AI systems that better understand and reason about complex visual and textual information, thereby enabling more accurate and context-aware responses. However, as with any advancement in AI, there are important ethical considerations: improved multi-modal models may be misused for generating misleading or biased content, or exacerbate privacy concerns when applied to sensitive visual data. We encourage the community to adopt responsible deployment practices, including transparency, fairness evaluation, and robust safeguards, to ensure that the benefits of such technologies are realized while minimizing potential harms. Ultimately, \method{} aims to foster AI systems that are not only more powerful but also aligned with human values and societal needs.

\end{document}